\documentclass[11pt]{article}
\usepackage[margin=1in]{geometry}

\usepackage{amsmath}
\usepackage{amsfonts}
\usepackage{bbm}
\usepackage{amsthm}
\usepackage{cancel}
\usepackage{hyperref}

\usepackage{thmtools} 
\usepackage{thm-restate}

\usepackage{color}
\definecolor{gray}{gray}{0.5}
\definecolor{lightred}{rgb}{1,0.6,0.6}
\definecolor{darkgreen}{rgb}{0,0.5,0}
\definecolor{darkblue}{rgb}{0.0,0.0,0.2}

\usepackage{natbib}
\setcitestyle{authoryear,open={[},close={]}} %

\DeclareMathOperator*{\argmax}{arg\,max}

\newcommand{\reals}{\mathbb{R}}
\newcommand{\E}{\mathop{\mathbb{E}}}

\newcommand{\D}{\mathcal{D}}

\newcommand{\mw}{{M^*_{\eta}}}
\newcommand{\ftrl}{{M_{R, \eta}}}

\newcommand{\R}{\mathcal{R}}

\newcommand\numberthis{\addtocounter{equation}{1}\tag{\theequation}}
\newtheorem{theorem}{Theorem}
\newtheorem{corollary}{Corollary}
\newtheorem{condition}{Condition}
\newtheorem{lemma}{Lemma}
\newtheorem{definition}{Definition}

\newtheorem{prop}{Proposition}

\newcommand{\Ubari}[2][]{\overline U_i#1(#2)}  %
\newcommand{\Ubarit}[2][]{\overline U_{it}#1(#2)} %
\newcommand{\UbariB}[2][]{\overline U_{iBt}#1(#2)} %

\newcommand{\condpit}{{p_{iBt}}}

\newcommand{\Btcomplement}{{\overline{B_t}}}

\title{ Forecasting Competitions with Correlated Events}
\author{ 
Rafael Frongillo, Manuel Lladser, Anish Thilagar, Bo Waggoner
\\
University of Colorado Boulder
}
\date{}

\begin{document}
\maketitle

\begin{abstract}
  Beginning with \citet{witkowski2101incentive}, recent work on forecasting competitions has addressed incentive problems with the common winner-take-all mechanism.
  \citet{frongillo2021efficient} propose a competition mechanism based on follow-the-regularized-leader (FTRL), an online learning framework.
  They show that their mechanism selects an $\epsilon$-optimal forecaster with high probability using only $O(\log(n)/\epsilon^2)$ events.
  These works, together with all prior work on this problem thus far, assume that events are independent.
  We initiate the study of forecasting competitions for correlated events.
  To quantify correlation, we introduce a notion of block correlation, which allows each event to be strongly correlated with up to $b$ others.
  We show that under distributions with this correlation, the FTRL mechanism retains its $\epsilon$-optimal guarantee using $O(b^2 \log(n)/\epsilon^2)$ events.
  Our proof involves a novel concentration bound for correlated random variables which may be of broader interest.
\end{abstract}

\section{Introduction}
Forecasting competitions, such as those on Kaggle or the Good Judgement project, attempt to discern which forecaster from a pool of contestants has the best forecasting model.
Traditionally, they ask each of the forecasters to predict the probability of some future events occurring, then observe those events and pick the forecaster with the largest empirical score as the winner.
As many have noted, this approach distorts the incentives of the forecasters, who will extremize their reports in order to increase their chances of having the maximum empirical score [\citealt{lichtendahl2007probability}, \citealt{kaggle2017march}, \citealt{witkowski2018incentive}, \citealt{aldous2019prediction}, \citealt{frongillo2021efficient}, \citealt{witkowski2101incentive}].
When their reports deviate from their models' true predictions, it becomes unclear which forecasters are actually the best at forecasting, as opposed to being better at strategizing, leaving no guarantee as to the quality of the winning model.

Several mechanisms have been proposed to address this problem. 
First, \cite{witkowski2018incentive} and \citet{witkowski2101incentive} proposed the Event Lotteries Forecasting (ELF) mechanism, which guarantees that it picks a good winner by solving the incentive problem. 
However, given $n$ forecasters, ELF requires $O(n \log(n) / \epsilon^2)$ events to pick an $\epsilon$-optimal forecaster, a large amount for public competitions or small data settings.
As an alternative, \citet{frongillo2021efficient} analyze Follow the Regularized Leader (FTRL) algorithms from online learning as forecasting competition mechanisms.
They show that FTRL is both \emph{approximately truthful} and \emph{accurate}: it will incentivize forecasters to report probabilities close to their beliefs, while also guaranteeing an $\epsilon$-optimal forecaster will win with high probability using only $O(\log(n)/\epsilon^2)$ events.

Thus far, however, all known mechanisms rely on a strong assumption: the events in question are independent.
By contrast, events in real world forecasting settings like elections, tournaments, and sequential processes are inherently correlated.
Swing states tend to swing in the same direction.
When a one-seed in a playoff bracket is eliminated early, the overall winner is likely to change.
If the price of a stock rises today, its expected price tomorrow will too.
Unfortunately, these previous mechanisms break under correlation (see below) and it is far from clear how to fix them.

Our main contribution is a more general high probability guarantee for FTRL which degrades gracefully in the presence of correlation.
Specifically, FTRL chooses an $\epsilon$-good forecaster with high probability given $O(b^2 \log(n)/\epsilon^2)$ events whose distribution has the $(b, \epsilon/20)$-block correlation property, which is formally defined in \S~\ref{sec:correlation}.
There are two keys parts to this proof.
In \S~\ref{sec:approx-truthful}, we show that FTRL is still approximately truthful under block correlation, so forecasters will not report values that are much different than their beliefs.
Then, in \S~\ref{sec:efficiency},  we combine this with a novel concentration inequality to show our main result, Theorem~\ref{thm:ftrl-accuracy}, that with high probability Multiplicative Weights efficiently chooses a forecaster with good beliefs.
\begin{restatable*}{theorem}{mainresult} \label{thm:ftrl-accuracy}
    For $\eta = \frac{\epsilon}{80 b} $, $\mw$ chooses an $\epsilon$-optimal forecaster with probability at least $1 - \delta$ if there are $m \geq m^* = \frac{400 b^2 \ln(8n/\delta)}{\epsilon^2}$ events with $(b, \frac{\epsilon}{20})$-block correlation.
    In other words, if forecasters only report undominated strategies, $\mw$ only requires $m^*$ events to choose an $\epsilon$-good forecaster.
\end{restatable*}
The concentration bound we present for block correlated distributions follows a somewhat complex argument that constructs a pair of connected martingales.
It is presented on its own in \S~\ref{sec:concentration-proof}.
We believe this approach may be of broader interest.

\subsection{Motivating Examples}
It has been said that there is only one form of independence, but infinite different forms of dependence.
Our goal is to define a measure of correlation that captures when the forecaster selection problem is feasible.
In this subsection, we walk through a few toy examples which motivate our measure of correlation.

\subsubsection*{Single Event}
To begin, let us recall the general incentive problem that forecasting competitions face.
When we have just a single event, there are only two possible outcomes, 0 or 1. 
Given just this single bit of information, mechanisms must decide which forecasters were the most accurate, and they tend to favor those with the most extreme beliefs.
A forecaster who thinks the bias is 0 seems to be the ``most right'' when the outcome is as well.
Concretely, imagine a forecaster who believes the outcome will be 1 with probability $0.5$, but also knows there are forecasters submitting predictions of $0.1$ and $0.9$.
Submitting $0.5$ would be foolish, as they could never win, so the optimal strategy is to extremize.
Given this extremizing behavior, there is no guarantee the mechanism is choosing a good forecaster.
Unless the mechanism is designed carefully, this incentive problem does not generally go away as the number of events increases \citep{witkowski2101incentive}.
Moreover, forecasters exploit this issue in real competitions \citep{kaggle2017march}.

\subsubsection*{Perfectly Correlated Events} \label{example:perfectly correlated}
Next, let us see why the problem is intractable when correlation is arbitrary.
Suppose we have a set of $m$ binary events that are perfectly correlated. 
Thus, their outcomes will either be all 0 or all 1.
With just the single observation, we are reduced to the single event setting where there is not enough information to pick a good forecaster, regardless of the size of $m$.

\subsubsection*{Disjoint Correlated Blocks} \label{example:disjoint-blocks}
Instead of all the $m$ events being perfectly correlated, suppose they are split into blocks of $b$ perfectly correlated events.
Now, we are essentially in the independent case with $m / b$ events, which we know is tractable for large enough $m$ using previously studied mechanisms.

What makes this case tractable while the previous one was not? 
Instead of having just one ``underlying'' event, we now have $m/b$ of them, and there is enough information to distinguish the forecasters.
One might expect that this difference is explained by the amount of randomness, e.g.\ Shannon entropy of the distributions, since this example has much more than the previous one.
However, we can construct distributions with relatively high entropy but only a few ``underlying'' events. 

\subsubsection*{Random Bias} \label{example:random-bias}
Suppose all the events are conditionally independent with the same probability $p$, but $p$ is chosen with equal likelihood from $\{1/4, 3/4\}$.
A priori, we will not know $p$, so the true probability of each event is $1/2$.
However, a perfect forecaster who reports $1/2$ for every event will always look worse than two competitors who each report all $1/4$ or all $3/4$.
This looks much like the single-event example, since there is only really one ``underlying'' event that matters, the choice of $p$.
Yet, the entropy of the outcome distribution is quite high, since each of the events individually have some randomness.
Entropy alone is not nuanced enough to distinguish these two scenarios.
We therefore seek a stricter property for the distributions that distinguishes the second and third example, which leads us to $(b, c)$-block correlation.

\subsection{Our measure: $(b,c)$-block correlation}
We say a distribution on events is $(b,c)$-block correlated if, for each event $t$, there are at most $b$ events that ``heavily influence'' it, such that knowledge of all other events together only change the probability of the event $t$ by at most $c$.
In other words, when we condition on a set of all other events except those $b$ influencers, the probability of event $t$ only changes by at most $c$.
The probability of event $t$ can change arbitrarily when conditioning on any subset of its influencer events, however.

In the previous random bias example, while there is a lot of randomness, the outcomes of the events are still all strongly coupled.
If we know the outcome of a small portion of them, we can probably guess which value of $p$ was chosen, and therefore know the bias of the remaining events with high confidence.
Specifically, the random bias example satisfies block correlation with $b = 1$ and $c = 1/4$.
Unless we let $b$ be very close to $m$, we cannot satisfy block correlation for any $c$ that is significantly smaller, since knowing just a small fraction of the events can give a strong indication of the choice of $p$.

In contrast, distributions with favorable block correlation have enough ``underlying'' events for the mechanism to work with by controlling how many events can strongly influence any one other event.
For instance, our disjoint block example above satisfies block correlation with $c=0$ and $b>0$, and can be thought to have $m/b$ underlying events.
Our block correlation condition generalizes this example in two respects: it does not require symmetry, i.e. an event can be contained in another's block but not vice versa, and it does not require complete independence outside of the block.

\section{Model}
We consider $n$ expert forecasters labeled by $[n] = \{1,\dots,n\}$ predicting the outcomes of $m$ binary events $Y_t \in \{0, 1\}$.
Let $Y = (Y_1, \dots, Y_m)$ be the random variable that is the vector of all the events' outcomes, drawn jointly from the distribution $\D$.
Let $\theta = (\E_{\D}[Y_1], \dots, \E_{\D}[Y_m])$ be the vector of marginal probabilities for each event.

\subsection{Beliefs and Accuracy}
Each forecaster has their own belief $\D_i$ of what they believe the distribution of $Y$ to be.
In particular, we define as $p_i = (\E_{\D_i}[Y_1], \dots, \E_{\D_i}[Y_m])$ as their belief of $\theta$.
While some previous works, such as \citet{witkowski2101incentive} consider more complex belief models, we assume that the belief distributions are immutable and independent of the reports of the other forecasters as in \citet{witkowski2018incentive} and \citet{frongillo2021efficient}.

To compare forecasters, we use their accuracy, defined as the averaged squared loss between their marginal belief vector $p_i$ and the true marginal probabilities $\theta$.
Specifically, the accuracy of forecaster $i$ is given by
\[ a_i = 1 - \frac{1}{m} \sum_{t=1}^m \left(p_{it} - \theta_t\right)^2 ~.\]
\citet{witkowski2101incentive} show this notion of accuracy is closely related to the quadratic score $S(x, y) = 1 - (x-y)^2$.
Up to a constant, each forecaster's accuracy is exactly the expected quadratic score between their marginal beliefs and $Y$.
\begin{lemma}[\citet{witkowski2101incentive}] \label{lem:accuracy-matches-expected-score}
Let $C_{\theta} = \frac{1}{m} \sum_t \theta_t\left(1 - \theta_t\right)$. Then,
\[ a_i = \frac{1}{m} \E_{\D} \left[\sum_{t=1}^m S(p_{it}, Y_t)\right] + C_\theta ~.\]
\end{lemma}

\subsection{Selection Mechanisms and Incentives}
A forecasting competition mechanism asks each forecaster to report a vector $r_i \in [0, 1]^m$ of what they believe $\theta$ to be.
If they are truthful, their reports $r_i$ would be exactly their beliefs $p_i$.
The mechanism then observes a sample $Y = y$ from $\D$, and uses $Y$ along with the reports to choose a forecaster.
Let $R$ be the $n \times m$ matrix of combined report vectors $(r_1, \dots, r_n)$.

\begin{definition}[\citet{frongillo2021efficient}]
  \label{def:mechanism}
   A \textbf{forecasting competition mechanism} $M$ is a family of functions $M_{n, m}:[0,1]^{n\times m}\times\{0,1\}^m \to \Delta_n$, for all $n,m \in \mathbb{N}$, where $M_{n, m}(R,y)_i$ is the probability with which the mechanism picks forecaster $i$ on reports $R$ and observed outcomes $y$.
   As $R$ determines $n$ and $m$, we suppress the subscripts.
  For a particular distribution $\D$, we write $M(R;\D) := \E_{y \sim \D}[ M(R,y)]$.
\end{definition}

If forecasters are non-strategic and simply report their beliefs, their accuracy can be approximated by their empirical score (Lemma~\ref{lem:accuracy-matches-expected-score}), and choosing a good forecaster is relatively straightforward.
In practice, however, forecasters are strategic and may manipulate their reports as a function of their beliefs in order to maximize the probability that they win.
To understand which mechanisms are robust to these manipulations, following \citet{frongillo2021efficient}, we will work with undominated and strictly dominated reports.
\begin{definition}
    \label{def:strictly-dominated}
    A report $\hat{r}_i$ strictly dominates $r_i$ if for all $R_{-i}$, $M(\hat{r}_i, R_{-i}; \D) > M(r_i, R_{-i}; \D)$.
    If such a $\hat{r}_i$ exists, then $r_i$ is strictly dominated.
    Otherwise, $r_i$ is undominated.
\end{definition}
If a forecaster reports an $r_i$ that is strictly dominated, then they would only increase their win probability by reporting $\hat{r}_i$ instead.
Therefore, strategic forecasters will only give undominated reports.
These reports are determined by the mechanism, since each forecaster is trying to maximize the probability the mechanism selects them. 
If the mechanism always keeps the undominated strategies close to a forecaster's beliefs, we call it approximately-truthful.
\begin{definition}[\citet{frongillo2021efficient}]
    \label{def:approx-truthful}
    A mechanism $M$ is $\gamma$-approximately truthful if for each $\D_i$,
    (i) there exists an undominated report $r_i$, and
    (ii) for all such $r_i$, $\|r_i - p_i\|_\infty < \gamma$.
\end{definition}

\subsection{Optimality and Efficiency}
A desirable mechanism will choose forecasters whose beliefs are close to the true distribution compared to the other forecasters.
In particular, we will analyze the probability that mechanisms choose an $\epsilon$-optimal forecaster, one whose accuracy is within $\epsilon$ of the best. 
\begin{definition}
\label{def:eps-optimal}
Forecaster $j$ is $\epsilon$-optimal if $a_{j} + \epsilon \geq \max_{i \in [n]} a_i$.
\end{definition}
Furthermore, we seek mechanisms which choose $\epsilon$-optimal forecasters with high probability under the assumption that forecasters do not submit dominated reports.
We call such mechanisms \emph{accurate}.
\begin{definition}
  \label{def:mech-accurate}
   A mechanism $M$ is \textbf{$(\epsilon,\delta)$-accurate} in the setting defined by $(n,m, \{\D_i\}_i,\D)$ if for all $R$ consisting of undominated reports, with probability at least $1-\delta$ over event outcomes $y \sim \D$ and $i \sim M(R,y)$, the winner $i$ is $\epsilon$-optimal.
\end{definition}
Finally, we will study how many events are needed for the accuracy guarantee to hold. 
A mechanism that always chooses $\epsilon$-optimal forecasters is not useful if it requires far more events to do so than we are able to observe.
We define a mechanism's \emph{event complexity} to be the minimum number of events needed for it to choose good forecasters with high probability.
In general, we seek mechanisms with a small event complexity.
\begin{definition}
  \label{def:event-complexity}
   The \textbf{event complexity} of a mechanism $M$ is the function $m^*: \mathbb{N} \times [0, 1] \times [0, 1] \rightarrow \mathbb{N}$ such that, for all $n,\epsilon,\delta$,
   the output $m = m^*(n,\epsilon,\delta)$ is the smallest integer such that, for all $(\{\D_i\}_i, \D)$, the mechanism $M$ is $(\epsilon,\delta)$-accurate in the setting $(n,m,\{\D_i\}_i, \D)$.
\end{definition}

\section{Correlation} \label{sec:correlation}
The accuracy guarantees of ELF and FTRL have only been established in settings where the events are all independently distributed.
We would like to extend the analysis of these mechanisms to settings with correlated events.
However, as we discuss in \S~\ref{example:perfectly correlated}, we cannot allow for arbitrary correlations. 
Recall that in the extreme case where all the events perfectly correlated, there are only two possible outcomes, and it becomes impossible to choose a good forecaster with high probability regardless of the number of events.
To proceed, we introduce a $(b,c)$-block correlation property to limit correlation both in the size $b$ of the ``block'' of events that any one event is strongly correlated with, and in the degree $c$ of correlation persisting outside of that block.

\begin{definition}\label{def:block-correlation}
  A distribution $D$ over $Y$ has $(b, c)$-block correlation if for each event $t$ there is a subset $B_t \subseteq [m]$ such that $|B_t| \leq b$ and
    \[ \left|\E_D[Y_t] - \E_{D}\left[Y_t | Y_{\Btcomplement} = y_{\Btcomplement}\right]\right| < c ~, \numberthis \label{eq:conditioned-belief-ineq}\]  
  for every $y_{\Btcomplement} \in \{0, 1\}^{m - |B_t|}$.
\end{definition}
This definition allows each event $t$ to be arbitrarily correlated with the at most $b$ other events in $B_t$.
While it may also be correlated with the events in $\Btcomplement$, the amount that correlation is restricted by requiring that for any outcome of the events in $\Btcomplement$ can only change the conditional mean of $Y_t$ by some small amount $c$.
In general we will need $c < \frac{1}{2}$ so we necessarily have $t \in B_t$.

\subsection{Example: Election Prediction}

To illustrate $(b,c)$-block correlation with a typical forecasting setting, consider the local elections for representatives within a country, say the United States.
Similar to the random bias example in \S~\ref{example:random-bias}, we can model each race as a coin flip between electing a Democrat and a Republican, each with some unknown bias $\hat \theta_t$, centered around some known mean $\theta_t$.
At the national level, there are usually some common factors that affect the overall voter turnout for each party, like the specific candidates running for president or divisive issues that are part of each party's platform.
We can think of this as a constant term that is added to the bias of all local election, as people across the country are influenced by these same factors.
This is captured by $c$ in our block-correlation condition, since conditioning on the outcomes of many other districts, we can estimate the bias by comparing the $\theta_t$ and $\hat \theta_t$.

In addition to national factors, there are local forces within each state that may draw people to the polls, such as specific ballot measures or state-level elections. 
These may cause larger changes in voter turnout, but they only influence voters that they directly apply to. 
For example, the turnouts within two adjacent districts in New York City are probably strongly coupled, but neither will likely affect how voters in Los Angeles behave.
Therefore, if we group the elections into blocks by state, with $b$ being the largest number of local elections in any state, we satisfy $(b,c)$-block correlation.

\subsection{FTRL}
While \citet{witkowski2101incentive} show that the truthfulness guarantees of ELF do not hold in the correlated setting, they hypothesize that it may be approximately truthful under mild correlations.
However, \citet{frongillo2021efficient} show that even in the independent setting, ELF has an event complexity that is $O(n \log(n) / \epsilon^2)$, while Multiplicative Weights' event complexity is only $O(\log(n)/\epsilon^2)$. 
In correlated settings, we expect ELF's performance to degrade, leaving FTRL as a more promising mechanism from the perspective of event complexity.

Given $\eta>0$ and a strictly convex differentiable function (the regularizer) $\R: \Delta_n \to \reals$, the FTRL mechanism $M_{\R,\eta}$ is given by
\begin{equation}
  \label{eq:FTRL}
  \ftrl(\R,\vec y) \in \argmax_{\pi\in\Delta_n}\left\{ \eta \sum_{i=1}^n \pi_i \sum_{t=1}^m S(r_{it},y_i) - \R(\pi) \right\}~.
\end{equation}
$M_{\R,\eta}(R,\vec y)$ is a singleton and can be written as
\begin{equation}
  \label{eq:FTRL-2}
  \ftrl(R,\vec y) = \nabla C(\eta \cdot q(R, y))~,
\end{equation}
where $C = \R^*$ is the convex conjugate of $\R$ and $q$ is the vector of the sums of quadratic scores for each forecaster given by  $q_i = \sum_{t=1}^m S(r_{it}, y_t)$ \citep[eq.~(4)]{frongillo2021efficient}.

When the events are independent and $C$ satisfies Condition~\ref{cond:regularizer} below, \citet{frongillo2021efficient} show that for $\eta < \min(\frac{\alpha}{2}, \frac{1}{\beta})$, FTRL is $(\beta + 1)\eta$-approximately truthful and has an event complexity of $m^* \leq \frac{5 \log(2 n/\delta)}{\eta \epsilon}$. 
Choosing $\eta = O(\epsilon)$ gives approximately truthful mechanisms with an event complexity of $O(\log(n/\delta)/\epsilon^2)$.
\begin{condition}[\citet{frongillo2021efficient}]\label{cond:regularizer}
  Given regularizer $\R$, let $C=\R^*$.
  Then $C$ is thrice differentiable, and:
  \begin{itemize}
      \item[(i)] There exists $\alpha>0$ such that $\partial^2_i C(x) \geq \alpha \, |\partial^3_i C(x)|$ for all $x\in\reals^n$ and $i\in \{1,\dots,m\}$.
      \item[(ii)] There exists $\beta>0$ such that $\log \left(\partial^2_i C(x)\right)$ is $\beta$-Lipschitz in $\|\cdot\|_{\infty}$ as a function of $x$, i.e. $\left|\log\frac{\partial^2_i C(x)}{\partial^2_i C(x')}\right| \leq \beta \|x - x'\|_{\infty}$.
       \item[(iii)] $\partial^2_iC(x) > 0$ for all $x$. (This follows from (ii).)
  \end{itemize}
\end{condition}
A well-known instance of FTRL is Multiplicative Weights, given by 
\begin{equation}
  \label{eq:mult-weights}
  \mw(R, y)_i = \frac{\exp\left(\eta \sum_{t=1}^m S(r_{it},y_t)\right)}{\sum_{j=1}^n \exp\left(\eta \sum_{t=1}^m S(r_{jt},y_t)\right)} ~,
\end{equation}
or equivalently by taking $C(x) = \log \sum_{i = 1}^n \exp(x_i)$.
Multiplicative Weights satisfies Condition~\ref{cond:regularizer} with $\alpha = 2$ and $\beta = 3$.

We will show that the FTRL accuracy and approximate truthfulness guarantees of \citet{frongillo2021efficient} extend to the block correlated setting.

\section{Approximate Truthfulness of FTRL} \label{sec:approx-truthful}
We begin by showing that FTRL remains approximately truthful in the presence of correlation.
For our analysis, we require that there are constants $b \geq 1, 0 \leq c < \frac{1}{2}$ independent of $i$ such that every $\D_i$ has $(b, c)$-block correlation.
This constraint implies $t \in B_t$ for all $t$ and $\D_i$.
We will often refer to $B_t$ as $t$'s block.

\subsection{Utilities}
Since each forecaster's objective is to be chosen by the mechanism, their utility is just their expected win probability.
Each forecaster seeks to maximize their utility as a function of their report, the only input they have on the mechanism.
Fixing the outcomes of the events $y$ and the reports of the other forecasters $r_{-i}$, forecaster $i$ has utility
\[U_i(r_i) = U_i(r_i; r_{-i}, y) = \ftrl(r_i, R_{-i}; y)_i ~.\]
Taking the expectation over $y$ then gives their expected utility
\[ \Ubari{r_i} := \E_{\D_i} \left[U_i(r_i; r_{-i}, Y) \right] ~. \]
We will also use versions of $U_i$ and $\overline U_i$ when we restrict attention to only round $t$.
Specifically, we define 
\[U_{it}(r_{it}; r_{-(it)}, y) = U_{it}(r_{it}; r_{i,-t}, r_{-i, -t}, y_{-t}, r_{-i,t}, y_t) = \ftrl(r_{it}; r_{i,-t}, r_{-i, -t}, r_{-i,t}, y)_i \]
as the utility of forecaster $i$ as a function of their report in round $t$, with all outcomes and other reports fixed.
We use the right hand side to isolate the variables for event $t$.
Now, we can define the expected utility fixing everything but a single event and it's corresponding reports as 
\[\Ubarit{r_{it}} = \E_{\D_{it}} \left[ U_{it}(r_{it}; r_{i,-t}, r_{-i,-t}, y_{-t}, r_{-i,t}, Y_t) \right] ~.\]
Since the events in a block are strongly correlated, it will also be necessary to restrict our attention to the report for a single event $t$, along with all the outcomes in its block. 
Taking the expectation over the outcomes $Y_{B_t}$ then gives us the expected utility
\[ \UbariB{r_{it}} = \E_{\D_i} \left[U_{it}(r_{it}; r_{i, -t}, r_{-i}, y_{\Btcomplement}, Y_{B_t}) \mid Y_{\Btcomplement} = y_{\Btcomplement}\right]  ~.\]

\subsection{Approximate truthfulness}
In the independent setting, \citet{frongillo2021efficient} use the strict concavity of $U_{it}$ to find each forecasters optimal reports.
In the correlated setting, we need to consider each block at a time, and thus we need strict concavity of $\overline{U}_{iBt}$.
\begin{lemma}\label{lem:FTRL-concave}
  Let $\R$ satisfy Condition~\ref{cond:regularizer} for $\alpha, \beta$.
  For $\eta < \tfrac{\alpha}{2}$, for all $i\in[n],t\in[m]$, and all $R_{-i}$, the functions $\UbariB{r_{it}}$ are strictly concave if $\D_i$ has $(b, c)$-block correlation.
\end{lemma}
\begin{proof}
    Fix any such $\eta$.
    \citet[Lemma 3]{frongillo2021efficient} show that every $U_{it}(r_{it})$ is strictly concave, so $\frac{d}{d r_{it}^2} U_{it}(r_{it}; r_{-(it)}, y) < 0$ for any fixed choice of $y$.

    By the Leibniz integral rule,
    \begin{align*}
        \frac{d^2}{d r_{it}^2} \UbariB{r_{it}} 
        &= \frac{d^2}{d r_{it}^2} \E_{\D_i} \left[U_{it}(r_{it}; r_{-(it)}, y_{\Btcomplement}, Y_{B_t}) \middle| Y_{\Btcomplement} = y_{\Btcomplement}\right] \\
        &= \E_{\D_i} \left[ \frac{d^2}{d r_{it}^2} U_{it}(r_{it}; r_{-(it)}, y_{\Btcomplement}, Y_{B_t}) \middle| Y_{\Btcomplement} = y_{\Btcomplement}\right] \\
        &< 0 ~.
    \end{align*}
    so $\UbariB{r_{it}}$ is strictly concave as well.
\end{proof}

Since $\UbariB{r_{it}}$ is strictly concave, it is uniquely maximized at a point.
Therefore, fixing everything but the outcomes of the events in $B_t$, forecaster $i$ has a unique best report $r^*_{it}$ for event $t$ that maximizes their expected utility. 
Next, we show that $r^*_{it}$ is close to $p_{it}$, so forecasters will be approximately truthful on event $t$ when the outcomes $Y_\Btcomplement$ are fixed.

\begin{lemma}\label{lem:FTRL-approx}
  Let $\R$ satisfy Condition~\ref{cond:regularizer} for $\alpha,\beta$ and let forecaster $i$'s belief $\D_i$ have $(b, c)$-block correlation.
  Fix all reports but $r_{it}$ and all outcomes but $Y_{\Btcomplement}$.
  Let $r^*_{it} = \argmax_{r_{it}\in[0,1]} \UbariB{r_{it}}$.
  Then for $0 < \eta < \min(\tfrac{\alpha}{2},\tfrac{1}{\beta b})$, 
  $|r^*_{it} - p_{it}| \leq \beta\eta b + (\beta\eta b)^2 + c \leq (\beta b + 1)\eta + c$.
\end{lemma}

\begin{proof}
    Fix any $Y_\Btcomplement = y_\Btcomplement$.
    Then, 
    \begin{align*} 
    \UbariB{r_{it}; r_{-(it)}, y_{\Btcomplement}, Y_{B_t}}  
    &= \E_{Y_{B_t} \sim \D_{i}} \left[\partial_i C\left(\eta q(r_{it}; r_{-(it)}, y_{\Btcomplement}, Y_{B_t}) \right) \middle| Y_\Btcomplement = y_\Btcomplement \right] ~.
    \end{align*}
  By Lemma~\ref{lem:FTRL-concave}, $\UbariB{r_{it}}$ is strictly concave, so it achieves its maximum at exactly the one point where its derivative vanishes.
  Therefore, $r^*_{it}$ will solve $\frac{d}{dr_{it}} \UbariB{r_{it}} = 0$. 
  Let $\condpit = \E_{D_{i}}[Y_t | Y_\Btcomplement = y_\Btcomplement]$.
  
  \begin{align*}
    0 
    &= \frac{d}{d r_{it}} \UbariB{r_{it}} \\
    &= \eta \E_{Y_{B_t} \sim \D_{i}} \left[\partial^2_i C\left(\eta q(r_{it}; r_{-(it)}, Y_{\Btcomplement}, Y_{B_t}) \right) S'(r_{it}, Y_t) \middle| Y_\Btcomplement = y_\Btcomplement \right]\\
    &= \eta \E_{y_{t} \sim \condpit} \left[ \E_{y_{(B_t \setminus t)} \sim \D_{i}} \left[\partial^2_i C\left(\eta q(r_{it}; r_{-(it)}, y_\Btcomplement, Y_{B_t}) \right) \middle| Y_t = y_t\right] S'(r_{it}, y_{t}) \middle| Y_\Btcomplement = y_\Btcomplement \right]\\
    &= \eta \sum_{y_{t} \in \{0, 1\}} \left[ \E_{y_{(B_t \setminus t)} \sim \D_{i}} \left[\partial^2_i C\left(\eta q(r_{it}; r_{-(it)}, y_\Btcomplement, Y_{B_t}) \right) \middle| Y_t = y_t, Y_\Btcomplement = y_\Btcomplement \right] \left(1 - \condpit - y_t(1-2\condpit)\right)\right] ~.\\
  \end{align*}

  Now, note that the first derivative of any weighted scoring rule will be 0 when the argument is exactly the normalized weights.
  Therefore, the previous equation will be solved by
  \begin{align*}
    r^*_{it} &= \condpit \tfrac
                        {\E_{Y_{B_t}} \left[\partial^2_i C\left(\eta q(r^*_{it}; Y_{B_t}) \right) \mid Y_t = 1\right]}
                        {(1 - \condpit) \E_{Y_{B_t}} \left[\partial^2_i C\left(\eta q(r^*_{it}; Y_{B_t}) \right) \mid Y_t = 0\right] + \condpit \E_{Y_{B_t}} \left[\partial^2_i C\left(\eta q(r^*_{it}; Y_{B_t}) \right) \mid Y_t = 1\right]} \\
    &= \condpit \left(\condpit + (1 - \condpit) \dfrac
                {\E_{Y_{B_t}} \left[\partial^2_i C\left(\eta q(r^*_{it}; Y_{B_t}) \right) \mid Y_t = 0\right]}
                {\E_{Y_{B_t}} \left[\partial^2_i C\left(\eta q(r^*_{it}; Y_{B_t}) \right) \mid Y_t = 1\right]} \right)^{-1} ~.
  \end{align*}
  Rearranging, we obtain
  \[
      \frac{\condpit}{r^*_{it}}
      = \condpit + (1 - \condpit) \dfrac
                {\E_{Y_{B_t}} \left[\partial^2_i C\left(\eta q(r^*_{it}; Y_{B_t}) \right) \mid Y_t = 0\right]}
                {\E_{Y_{B_t}} \left[\partial^2_i C\left(\eta q(r^*_{it}; Y_{B_t}) \right) \mid Y_t = 1\right]} ~. \numberthis \label{eqn:log-p-r-ratio}
  \]

  As $S$ is bounded in $[0, 1]$, we have $\|\max_{y_{B_t}} q(r^*_{it}; y_{B_t}) - \min_{y_{B_t}} q(r^*_{it}; y_{B_t}) \|_\infty \leq |B_t| \leq b$. 
  So by Condition \ref{cond:regularizer}(ii), 
  \begin{align*}
    \left| \log \dfrac
                {\max_{y_{B_t}} \left[\partial^2_i C\left(\eta q(r^*_{it}; y_{B_t}) \right)\right] } 
                {\min_{y_{B_t}} \left[\partial^2_i C\left(\eta q(r^*_{it}; y_{B_t}) \right)\right] } \right|
     & \leq \beta \eta b ~.
  \end{align*}
  Since $\condpit + (1-\condpit)x \leq x$ for $x \geq 1$, and $x \leq \condpit + (1-\condpit)x$ for $x \leq 1$, maximizing and minimizing the logarithm of the right hand side of eq.~\eqref{eqn:log-p-r-ratio} yields
  \begin{align*}
         \left| \log \left( \frac{\condpit}{r^*_{it}} \right)\right|
      & \leq \left| \log \left( \condpit + (1 - \condpit) \dfrac
                { \max_{y_{B_t}} \left[\partial^2_i C\left(\eta q(r^*_{it}; y_{B_t}) \right)\right] } 
                {\min_{y_{B_t}} \left[\partial^2_i C\left(\eta q(r^*_{it}; y_{B_t}) \right)\right] } \right)\right| \\
      & \leq \left| \log \left( \dfrac
                { \max_{y_{B_t}} \left[\partial^2_i C\left(\eta q(r^*_{it}; y_{B_t}) \right)\right] } 
                {\min_{y_{B_t}} \left[\partial^2_i C\left(\eta q(r^*_{it}; y_{B_t}) \right)\right] } \right)\right| \\
      & \leq \beta \eta b ~.
  \end{align*}  
  For $x \in [0, 1]$, $e^{x} < 1 + x + x^2$.
  So $\frac{\condpit}{r^*_{it}}$ and $\frac{r^*_{it}}{\condpit} \leq 1 + \beta \eta b + \left(\beta \eta b\right)^2$. Then, $|r^*_{it} - \condpit | \leq \beta \eta b + \left(\beta \eta b\right)^2$.
  By eq.~\eqref{eq:conditioned-belief-ineq}, $|\condpit - p_{it}| < c$.
  Therefore by the triangle inequality, $|r^*_{it} - p_{it} | \leq \beta \eta b + \left(\beta \eta b\right)^2 + c \leq (\beta b + 1)\eta + c$.
\end{proof}

Since the events of $\Btcomplement$ are less correlated with $t$, conditioning on their outcomes should not influence the optimal report $r^*_{it}$ much, so we obtain approximate truthfulness when conditioning on all events as well.
\begin{theorem}\label{thm:ftrl-approx-truthful}
  Let $\R$ be any regularizer satisfying Condition~\ref{cond:regularizer} with $\alpha,\beta>0$.
  Let $\gamma = (\beta b + 1)\eta + c$.
  Then $\ftrl$ is $\gamma$-approximately truthful for any $\eta < \min(\tfrac{\alpha}{2},\tfrac{1}{\beta b})$ if the distributions $\D_i$ have $(b, c)$-block correlation.
\end{theorem}
\begin{proof}

    \newcommand{\rplusit}{r^+_{it}}
    \newcommand{\rminusit}{r^-_{it}}

    Fix any such $\eta$.
    We will show that any report that is not $\gamma$-approximately truthful is strictly dominated.
    For any forecaster $i$, choose any $\hat r_{i}$ such that $|\hat r_{it} - p_{it}| > \gamma$ for some $t$.
    Let $\rplusit = p_{it} + \gamma$ and $\rminusit = p_{it} - \gamma$.
    We either have $\hat r_{it} < \rminusit$ or $\hat r_{it} > \rplusit$.

    By Lemma \ref{lem:FTRL-concave} we know that for every choice of $y_{\Btcomplement}$, $\UbariB{\cdot}$ is strictly concave, and furthermore by Lemma \ref{lem:FTRL-approx}, it achieves its maximum in $(\rminusit, \rplusit)$.
    Therefore, for all $y_{\Btcomplement}$, if $\hat r_{it} < \rminusit$, then $\UbariB{r_{it}} < \UbariB{\rminusit}$ and if $\hat r_{it} > \rplusit$, then $\UbariB{r_{it}} < \UbariB{\rplusit}$.

    If $\hat r_{it} < \rminusit$, let $r'_i = (\hat r_{i1}, \dots, \hat r_{i(t-1)}, \rminusit, \hat r_{i(t+1)}, \dots, \hat r_{im})$, the vector of reports obtained by replacing the $t$th entry of $\hat r_{i}$ with $\rminusit$.
    Then,
    \begin{align*}
        \Ubari{\hat r_i; r_{-i}, Y}
        &= \E_{\D_i} \left[U_i(\hat r_i; r_{-i}, Y) \right] \\
        &= \E_{y_{\Btcomplement} \sim \D_i} \left[ \E_{Y_{B} \sim \D_{i}} \left[U_i(\hat r_i; r_{-i}, y_{\Btcomplement}, Y_{B_t}) \middle| Y_{\Btcomplement} = y_{\Btcomplement}\right]  \right]\\
        &= \E_{y_{\Btcomplement} \sim \D_i} \left[ \UbariB{\hat r_{it}; \hat r_{i(-t)}, r_{-i}, y_{\Btcomplement}, Y_{B_t}} \right]  \\
        &< \E_{y_{\Btcomplement} \sim \D_i} \left[ \UbariB{\rminusit; \hat r_{i(-t)}, r_{-i}, y_{\Btcomplement}, Y_{B_t}} \right]  \\
        &= \E_{y_{\Btcomplement} \sim \D_i} \left[ \E_{Y_{B} \sim \D_{iB}} \left[U_i(r'_i; r_{-i}, y_{\Btcomplement}, Y_{B_t}) \middle| Y_{\Btcomplement} = y_{\Btcomplement} \right]  \right]\\
        &= \E_{\D_i} \left[U_i(r'_i; r_{-i}, Y) \right] \\
        &= \Ubari{r'_i; r_{-i}, Y} ~.
    \end{align*}
    Similarly, if $\hat r_i > \rplusit$, we let $r'_i = (\hat r_{i1}, \dots, \hat r_{i(t-1)}, \rplusit, \hat r_{i(t+1)}, \dots, \hat r_{im})$, and obtain $\Ubari{\hat r_i; r_{-i}, Y} < \Ubari{r'_i; r_{-i}, Y}$.
    In either case, $r'_i$ strictly dominates $\hat r_i$.
    Therefore, any $\hat r_i$ such that $\|\hat r_i - p_i\|_\infty > \gamma$ is strictly dominated, so $\ftrl$ is $\gamma$-approximately truthful.
\end{proof}
\begin{corollary}\label{cor:mw-approx-truthful}
    $\mw$ is $(4 b \eta + c)$-approximately truthful for any $\eta < \frac{1}{4b}$ if the $\D_i$ have $(b, c)$-block correlation.
\end{corollary}

\section{Efficiency of Multiplicative Weights} \label{sec:efficiency}
The approximate-truthfulness of $\mw$ guarantees that the forecasters reports reflect their beliefs.
Next, we will show that this approximate truthfulness implies that the mechanism is accurate and efficient: it will choose a winner with $\epsilon$-good beliefs with high probability, and it only requires $O(b^2 \log(n)/\epsilon^2$ events to do so.
This result, which we prove in Theorem~\ref{thm:ftrl-accuracy}, requires two intermediate results. 
Lemma~\ref{lem:score-concentration} shows that with high probability, forecasters' empirical scores will match their expected scores.
Lemma~\ref{lem:MW-eps-optimal} shows that with high probability, Multiplicative Weights will choose a forecaster with a large empirical score.
The proof of Theorem~\ref{thm:ftrl-accuracy} then follows by taking a union bound over both those events occurring and carefully choosing $\eta$.

Both \citet{witkowski2101incentive} and \citet{frongillo2021efficient} require that, in addition to the forecasters believing the events were independent, the true distribution of the events adhered to that independence as well.
Similarly, in addition to requiring each of the $\D_i$ to have $(b, c)$-block correlation, we assume that the true distribution of the outcomes $\D$ does as well.

Recall that $q_i = \sum_t S(r_{it}, y_t)$ is the total quadratic score of forecaster $i$ and that $\mw$ takes the $q_i$ as inputs and uses them to choose a winner $w$ according to the distribution given in eq.~\eqref{eq:mult-weights}.
\citet{frongillo2021efficient} showed that with high probability, $q_w$ will be close to $\max_i q_i$.
Their proof does not depend on $\D$, since it is simply a property of the mechanism, so it applies in the block correlated case as well.

\begin{lemma}[\citet{frongillo2021efficient}] \label{lem:MW-eps-optimal}
    With probability at least $1 - \frac{\delta}{2}$, the winner $w \in [n]$ chosen by $\mw$ satisfies $q_{w} \geq \max_i q_i - \frac{\log(2n/\delta)}{\eta}$.
\end{lemma}

This means that Multiplicative Weights chooses a forecaster with good reports, but we want to show that it will choose a forecaster with a high accuracy.
Let $q^*_i = \E_{\D} \left[\sum_t S(p_{it}, y_t)\right] = m (a_i - C_\theta)$ be the quadratic scores of each forecaster's beliefs, where the final equality follows from Lemma~\ref{lem:accuracy-matches-expected-score}. 
By approximate truthfulness and the Lipshitz properties of the quadratic score, $q^*_i$ must be close to $\E_{\D} [q_i] = \E_{\D} \left[\sum_t S(r_{it}, y_t)\right]$, the expected score of forecaster $i$'s reports. 
Therefore, if we can show $q_i$ is close to its mean $\E_\D[q_i]$ with high probability, then it will also be close to $q^*_i$.

In the independent setting, \citet{frongillo2021efficient} show that the $q_i$ concentrate around their means using a straightforward Hoeffding bound.
In a correlated setting, this approach will not work, as Hoeffding relies heavily on independence. 
Instead, we develop a novel concentration bound for block correlated distributions, given by Theorem~\ref{thm:martingale}.
This bound is complex and interesting in its own right, so we defer its presentation and proof to \S~\ref{sec:concentration-proof}.
\begin{lemma} \label{lem:score-concentration}
  If $\D$ has $(b, c)$-block correlation, then for any forecaster $i$,
  \[ \Pr\left[ \left|q_i - \E_\D[q_i]\right| \geq mc + 2b\sqrt{2 m \ln(8n/\delta)} \right] \leq \frac{\delta}{2n} ~. \]
\end{lemma}
\begin{proof}
  For all $i\in[n]$, we have
    \begin{align*}
        |q_i - \E_\D[q_i]|
        &= \left|\sum_{t} S(r_{it}, y_t) - \E_\D\left[\sum_{t} S(r_{it}, Y_t)\right]\right| \\
        &= \left|\sum_{t} \left(1 - r^2_{it} - y^2_t + 2 r_{it} y_t - \E_\D\left[1 - r^2_{it} - Y^2_t + 2 r_{it} Y_t\right] \right)\right| ~, \\
    \intertext{and as $y_t \in \{0, 1\}$, $y_t = y_t^2$, and $|2r_{it} - 1| \leq 1$,}
        &= \left|\sum_{t} (2 r_{it} - 1) \left(y_t - \E_\D\left[Y_t\right] \right) \right| \\
        &\leq \left|\sum_{t} y_t - \sum_t \E_\D\left[Y_t\right] \right| ~.
    \end{align*}
Since $\D$ is $(b, c)$-block correlated, by Theorem~\ref{thm:martingale},
    \[
        \Pr\left[ \left|\sum_t y_t - \sum_t \E_\D[Y_t]\right| \geq mc + 2b\sqrt{2 m \ln(8n/\delta)} \right]
        \leq \frac{\delta}{2n} ~.
    \]
Substituting the first inequality into the second completes the proof.
\end{proof}

Combining the previous results and choosing $\eta$ carefully, we can show that Multiplicative Weights efficiently chooses an accurate forecaster even in the presence of correlation.
\mainresult

\begin{proof}
    Fix any such $\eta$ and $m$.
    Let $w$ be the forecaster chosen by $\mw$, $j = \argmax_i a_i$ be the most accurate forecaster, and $F$ be the set of $\epsilon$-bad forecasters.
    We want to show that $w \not\in F$ with high probability.
    By Lemma~\ref{lem:accuracy-matches-expected-score},
    \[F = \{i \mid a_i + \epsilon < a_j\} = \{i \mid \left(q^*_j - q^*_i\right)  > m \epsilon\} ~. \numberthis \label{eqn:eps-bad-set} \] 
    
    Fix any $i \in F$.
    By Corollary~\ref{cor:mw-approx-truthful}, $\mw$ is $4b\eta + \frac{\epsilon}{20} = 8b\eta$-approximately truthful, so $|r_{it} - p_{it}| < 8b\eta$ for all $t$.
    Since the quadratic score $S$ is 2-Lipshitz and all $0 \leq r_{it}, y_t \leq 1$, $|q^*_i - \E[q_i]| < 16b\eta m = \frac{m \epsilon}{5}$.
    Similarly, $|q^*_j - \E[q_j]| < \frac{m\epsilon}{5}$.
    Applying this to eq.~\eqref{eqn:eps-bad-set}, for any $i \in F$, 
    \[\E[q_j] - \E[q_i] > \frac{3 m \epsilon}{5} ~. \numberthis \label{eqn:expected-score-bound}\]

    Using our choice of $m$, by Lemma~\ref{lem:score-concentration}, with probability $1 - \frac{\delta}{2n}$, $|q_i - \E[q_i]| < \frac{m \epsilon}{20} + 2b \sqrt{2 m \ln(8n/\delta)} \leq \frac{m\epsilon}{5}$.
    Similarly, with probability $1 - \frac{\delta}{2n}$, $|q_j - \E[q_j]| < \frac{m \epsilon}{5}$.
    
    $|F| \leq n-1$, so taking the union bound over $j$ and all possible $i \in F$ and using eq.~\eqref{eqn:expected-score-bound}, $q_j - q_i > \frac{m \epsilon}{5}$ for all $i$ with probability $1 - \frac{\delta}{2}$.
    On the other hand, by Lemma~\ref{lem:MW-eps-optimal}, with probability $1 - \frac{\delta}{2}$, $q_{j} - q_{w} < \frac{\log (2n/\delta)}{\eta} < \frac{m \epsilon}{5}$. 
    Taking a union bound over both those events, we have that with probability $1-\delta$, all the bad forecaster's have empirical scores $< q_j - \frac{m\epsilon}{5}$, and the chosen winner has an empirical score $> q_j - \frac{m\epsilon}{5}$, and therefore no $\epsilon$-bad forecaster is chosen.
    Therefore, $\mw$ chooses an $\epsilon$-optimal forecaster with probability at least $1 - \delta$. 
\end{proof}

\section{Concentration Bound} \label{sec:concentration-proof}
The main result of this section is the following concentration bound for sums of $(b,c)$-block correlated random variables.
Several concentration bounds for correlated random variables have already appeared in the literature \citep{fan2015exponential, pelekis2017hoeffding}.
However, they do not allow for the strong correlation that we allow for within each event's block.

\begin{restatable}{theorem}{martingalethm}
\label{thm:martingale}
Let $Z_1,\ldots,Z_m$ be possibly dependent $[0,1]$-valued random variables. For each $i\in[m]$, let $B_i\subset[m]\setminus\{i\}$ and $\bar B_i:=[m]\setminus(B_i\cup\{i\})$. Define  $\beta_i:=(Z_j:j\in B_i)$ and $\bar\beta_i:=(Z_j:j\in\bar B_i)$. If there is an integer $b\ge1$ such that, for all $i\in[m]$, $|B_i|\le b-1$ and a constant $c\ge0$ such that $\left| \E[Z_i\mid\bar\beta_i] - \E[Z_i] \right| \leq c$ then
\[\Pr\left[ \left|\sum_{j=1}^mZ_j - \sum_{j=1}^m\E[Z_j]\right| \geq mc + 2b\sqrt{2m\ln(4/\delta)} \right] \leq \delta,\]
for all $0<\delta\le1$.
\end{restatable}

Here, we explain the key idea of the construction that allows us to prove this result.
In Appendix \ref{appendix-martingale proof}, we give the complete proof.

\paragraph{The standard martingale approach.}
First, recall (e.g. \cite{mitzenmacher2005probability}) that the standard Azuma-Hoeffding proof of concentration for martingales goes as follows.
We are given a sum $S = X_1 + \cdots + X_m$ and an assumption of centered, bounded differences, i.e. $\E[X_i \mid X_1,\dots,X_{i-1}] = \mu_i$ and $|X_i - \mu_i| \leq t$.
We use the usual ``Chernoff'' method where the key step is to bound the moment generating function $\E\left[ e^{\lambda S} \right]$ for all $\lambda > 0$.
Here,
\begin{align}
  \E \left[ e^{\lambda S} \right]
  &=    \E\left[ \prod_{i=1}^m e^{\lambda X_i} \right]  \label{eqn:azuma-1} \\
  &=    \E\left[ \prod_{i=1}^{m-1} e^{\lambda X_i} \E\left[ e^{\lambda X_m} \middle| X_1,\dots,X_{m-1} \right] \right]  \label{eqn:azuma-2} \\
  &\leq \E\left[ \prod_{i=1}^{m-1} e^{\lambda X_i} \left( e^{\lambda \mu_m + \lambda^2 t^2/2}\right) \right]  & \text{Hoeffding's Lemma}  \label{eqn:azuma-3} \\
  &=    e^{\lambda \mu_m + \lambda^2 t^2/2} \E\left[ \prod_{i=1}^{m-1} e^{\lambda X_i} \right]  . \label{eqn:azuma-4}
\end{align}
By repeating the argument, we ``peel off'' $X_{m-1},\dots,X_1$ one at a time and obtain a bound of the form $\E\left[ e^{\lambda S} \right] \leq e^{\lambda \sum_i \mu_i + m \lambda^2 t^2/2}$, which leads to the standard Azuma-Hoeffding bound.

\paragraph{The role of $c$.}
The impact of $c$ in the argument is straightforward.
In particular, with $(1,c)$-block correlation, the above argument goes through for $X_i = Z_i$ directly, with the small change that $\E[X_i \mid X_1,\dots,X_{i-1}] \in \mu_i \pm c$.
This change simply results in an additional $mc$ added to the range of the final sum, i.e. $\E[e^{\lambda S}] \leq e^{\lambda \sum_i \mu_i + \lambda mc + m \lambda^2 t^2/2}$.
For example, with a sum of $m$ Booleans, each of whose mean is in $[\tfrac{1}{2}+c,\tfrac{1}{2}-c]$ conditioned on any realizations of the others, the sum will concentrate to within $[\tfrac{1}{2}m - mc - O(\sqrt{m}), \tfrac{1}{2}m + mc + O(\sqrt{m})]$.
More concretely, the random bias example (\S~\ref{example:random-bias}) has $(1,\tfrac{1}{4})$-block correlation, and follows this behavior.

\paragraph{The role of dependence.}
With $(b,c)$-block correlation for $b \geq 2$, Equation \ref{eqn:azuma-2} no longer holds with $X_i=Z_i$.
We need to condition only on variables in $\overline B_m$, the ``non-influencers'', rather than conditioning on all of $X_1,\dots,X_{m-1}$.
For example if $B_m = \{m-1,m\}$, Equation \ref{eqn:azuma-2} would be replaced with
  \[ \E\left[ \prod_{i=1}^{m-2} e^{\lambda Z_i} \E\left[ e^{\lambda (Z_{m-1} + Z_m)} \middle| Z_1,\dots,Z_{m-2} \right] \right]  . \]
Now, however, we are stuck.
Although $Z_m$ is controlled, because we have brought its influencer $Z_{m-1}$ inside the expectation, $Z_{m-1}$ is not controlled.
Suppose, for example, that $B_{m-1} = \{m-1,m-2\}$.
Then $Z_{m-1}$ may depend arbitrarily on $Z_{m-2}$.

A useful trick, inspired by H\'ajek projection~\cite{Van98}, is to introduce additional copies of the random variables.
Intuitively, we are willing to expand each random variable $Z_i$ into a ``mega-variable'' bounded in $[-b,b]$, because it correlates arbitrarily with $b-1$ other random variables.
So we can introduce an entire additional copy of every ``influencer'' variable in $B_i$, if it gives us enough conditional independence structure to apply a martingale argument.

We will do so, but very carefully.
To explain our approach, visualize the simple martingale independence structure above as
 \[ X_1 + \cdots + X_{m-1} ~\Big|~ X_m , \]
denoting that $X_m$ is controlled when conditioning on $X_1,\dots,X_{m-1}$.
Now, in the simple example where $B_m = \{m-1,m\}$, we were forced to move $Z_{m-1}$ into the conditional.
This looks like the following structure:
 \[ Z_1 + \cdots + Z_{m-2} ~\Big|~ Z_{m-1} + Z_m , \]
but it is not true that $Z_{m-1} + Z_m$ is controlled conditioned on $Z_1,\dots,Z_{m-2}$.
More precisely, $Z_{m-1} + Z_m$ is bounded, but we cannot control its conditional mean.
A further problem is that when we ``peel off'' $Z_m$, we expect to be left with $Z_1 + \cdots + Z_{m-1}$, but there is no $Z_{m-1}$ on the left side.

To fix both problems at once, we draw a new variable $Z_{m-1}'$ independently \emph{conditioned on $Z_1,\dots,Z_{m-2}$}.
We add $Z_{m-1}'$ to the sum and subtract its conditional expectation:
 \[ Z_1 + \cdots + Z_{m-2} + Z_{m-1}' ~\Big|~ \underbrace{Z_{m-1} + Z_m - \E[Z_{m-1}' \mid Z_1,\dots,Z_{m-2}]}_{\stackrel{\text{def}}{=} X_m} . \]
The right side is now defined to be the ``mega-variable'' $X_m$.
We have both that $X_m$ is bounded and that its mean, conditioned on the left side, is equal to the mean of $Z_m$ conditioned on the same.
We mention an important subtlety: it is not possible to exchange the locations of $Z_{m-1}'$ and $Z_{m-1}$, because then $Z_m$ would no longer be controlled conditioned on the left side, which would contain its influencer $Z_{m-1}$.
The construction works because $Z_{m-1}'$ is not an influencer, since it is drawn independently of $Z_m$ and $Z_{m-1}$ conditioned on $Z_1,\dots,Z_{m-2}$.

We also mention that the above construction is specific to the structure $B_m = \{m-1,m\}$, whereas in general, we will need to bring all influencers in $B_m$ to the right side, redraw all of them jointly conditioned on the left side, and subtract off all conditional means.

It is now possible to ``peel off'' $X_m$ in a martingale argument, noting that it is distributed as $Z_m$ plus a bounded, conditionally mean-zero quantity.
And crucially, the left side variables are distributed identically to $Z_1,\dots,Z_{m-1}$.
So we can treat them for purpose of analysis as a fresh copy of $Z_1,\dots,Z_{m-1}$, construct $X_{m-1}$ from $Z_{m-1}$ and its influencers in $B_{m-1}$, peel off $X_{m-1}$, and repeat.
A minor point is that we ignore influencers that have already been peeled off.
For example, if $m \in B_{m-1}$, we simply drop it from the set and only consider the other influencers of $Z_{m-1}$, marginalizing over $Z_m$.

\paragraph{Tying the argument together.}
In the end, this construction gives us a martingale with the necessary structure,
 \[ X_1 + \cdots + X_{m-1} ~\Big|~ X_m , \]
whose expectation is equal to the expectation of the original sum.
This implies that the new sum $X_1 + \cdots + X_m$ concentrates around the expectation of the original sum.
We then need a second result, namely, that the new sum concentrates around the \emph{realization} of the original sum.
Putting these together implies that the realization of the original sum concentrates around its expectation.

To get the second result, note that above, we added the quantity $Z_{m-1}' - \E[Z_{m-1}' \mid Z_1,\dots,Z_{m-2}]$ to the original sum.
Let us call that $E_m$.
Repeating, we get that the difference between the original sum and $X_1 + \cdots + X_m$ is given by $E_1 + \cdots + E_m$.
It is crucial that $E_1 + \cdots + E_m$ also forms a martingale sum, which requires a bit more nontrivial analysis to confirm.

\paragraph{Tightness of the bound.}
A simple example shows that the tail bound is loose by at most a $\sqrt{b}$ factor.
Consider $m \gg b$ Boolean variables with marginal means $\tfrac{1}{2}$.
They are divided into $m/b$ groups, and each group is perfectly correlated, i.e. either all are zero or all are one.

First, a hidden fair coin is tossed.
If heads, then every variable has marginal mean $\tfrac{1}{2}+c$.
If tails, then every variable has marginal mean $\tfrac{1}{2}-c$.
This example exhibits almost exactly $(b,c)$-block correlation, because conditioning on all variables outside of the group, the value of the hidden coin becomes apparent with high probability, revealing whether members of the group are biased up or down by $c$.

The sum can be analyzed as $m/b$ variables, each in $\{0,b\}$, where with equal probability all have mean $(\tfrac{1}{2}+c)b$ or $(\tfrac{1}{2}-c)b$. 
Conditioning on the coin flip, the sum has variance $\Theta(mb)$, as the sum of $m/b$ independent variables of variance $\Theta(b^2)$.
In this case, the tightest \emph{a priori} tail bound is of the form $\Pr[|S - \E[S]| \geq mc + O(\sqrt{m b \ln(1/\delta)})] \leq \delta$.
However, in Theorem \ref{thm:martingale}, the $b$ is outside the square root.
It remains to be seen if Theorem \ref{thm:martingale} can be tightened, or if there exists an example of $(b,c)$ correlation with higher variance.

\section{Discussion}

We conclude with some conceptual points about our model, and future work.

\subsection{Nuances of ground truth}
Recall that in the single event example in \S~\ref{example:perfectly correlated}, the only event is a single fair coin flip who's outcome is either 0 or 1 with probability $\tfrac{1}{2}$. 
There are a few ways to think of what that probability represents.
In our model, the probability $\tfrac{1}{2}$ captures inherent randomness in the world; both outcomes could happen.
So if a forecaster reports 1, we consider them to be a poor forecaster, since they are very far from the true probability. 
Alternatively, there is the deterministic view that there are two universes, one where the coin is 0, and the other where it is 1, and the bias of the coin comes from our uncertainty about which universe we are in.
A forecaster who predicts 1 may claim to know that we are in the second universe, and therefore should be considered the best.
This difference is essentially whether we consider this information ``knowable'' or whether the truth is inherently random (cf.\ aleatoric vs epistemic uncertainty).

In the random bias example (\S~\ref{example:random-bias}), we had a set of of events whose true bias was chosen to be $\tfrac 1 4$ or $\tfrac 3 4$ with equal probability.
We consider the best forecaster to be one whose belief is $p_{it} = \tfrac{1}{2}$ for all $t$, since this is the marginal probability of each event before the bias is chosen, while we consider a forecaster who thinks the true bias is $\tfrac 1 4$ to be much worse.
As in the single event case, however, it is possible that they have some external signal or expertise and is certain (and correct) about the bias.
Fundamentally, our high probability guarantees rely on the true distribution $\D$ capturing all information which is knowable, and leaving only randomness which is inherent to the world.

\subsection{Measuring forecaster accuracy}
We measure the accuracy of forecasters by summing the accuracy of their marginal probability of each event.
This measure is natural when events are independently distributed, but it is less clear that it is the right measure in the presence of correlation.
Under the measure we adopt, if we duplicate an event to two perfectly correlated ones, the relative contribution of that event to a forecaster's accuracy increases.
Consequently, in a disjoint block setting (\S~\ref{example:disjoint-blocks}) where blocks may have different sizes, it would perhaps be better to normalize within each block to weigh them evenly.
On the other extreme, there are other well-known metrics to compare the entire joint distributions over events, like KL-divergence or total variation distance, yet these would seem to require forecasters to report $2^m$ probabilities.
It would be interesting to explore the space in between, toward accuracy measures that more natively capture correlation but remain tractable.

\subsection{Future work}
While we show that FTRL mechanisms are approximately truthful, we only show that Multiplicative Weights is efficient.
It remains to be seen if these results apply to FTRL with regularizers other than negative entropy.
We would also like to know if there are practical (exactly) truthful mechanisms that are robust to correlation.
While one could ask forecasters to report a joint distribution on all events (as a single ``meta event'') and then use the 1-event version of ELF (which \citet{witkowski2101incentive} show is still truthful), this mechanism is far from practical in many respects.

Furthermore, while our results show that block correlation is a sufficient condition for efficient forecasting mechanisms, we do not know to what extent it is necessary.
It is possible that there are distributions with block correlation parameters that are unfavorable for our analysis, yet still allow for efficient mechanisms.
It would be interesting to see what those distributions look like, and if there is a more general property that encompasses those distributions as well as the block correlated ones that work well with our results.

The belief model we adopt could be generalized in several respects.
We require that all forecasters believe the true distribution is $(b, c)$-block correlated for the same constants $b$ and $c$.
In true forecasting competitions, the participants can have a very wide range of beliefs.
Even if the true distribution is block correlated, it is possible that there is an extremely misinformed forecasters that believes all the events are always perfectly correlated.
It is not clear if the mechanisms we analyzed are robust to such cases. 
Finally,
\citet{witkowski2101incentive} use a Bayesian model where forecasters may believe that the reports of their competitors are correlated with the truth.
For example, when forecasting the weather, we expect meteorologists to make better predictions than most average citizens.
We believe that our results will extend to this setting as well.

\subsection{Sequential predictions}
Though \citet{frongillo2021efficient} show the robustness of FTRL in the strategic online setting as well, we focus only on the offline setting.
In their online setting, as in \citet{dawid2020learnability} and \citet{choe2021comparing}, forecasters see the outcome of each event before predicting the subsequent one.
They can therefore update their belief by conditioning on the events that have already occurred, making their subsequent prediction independent of them, and thus alleviating some of the challenges of correlated mechanisms.
The guarantees of \citet{frongillo2021efficient} in the online setting still hold when accuracy is defined over those conditional distributions.

\subsection*{Acknowledgements}
We thank
Eduardo Corona, 
Rupert Freeman, 
Ziyu Li,
David Pennock, 
Aaditya Ramdas,
Ambuj Tewari, 
and 
Jens Witkowski, 
for their helpful ideas, suggestions, and comments.
This material is based upon work supported by the National Science Foundation under Grant No. IIS-2045347.

\bibliography{correlation}

\appendix

\section{Proof of Theorem~\ref{thm:martingale}} \label{appendix-martingale proof}

\martingalethm*

\begin{proof}

Define
\[S:=\sum_{i=1}^mZ_i.\]
The key idea for this is to define recursive a sequence $\hat S_m,\ldots,\hat S_0$ such that $\hat S_0$ is a proxy of $S$ for which we can show concentration. 
Ideas from H\'ajek's projection method~\cite{Van98} underlie our construction.

\begin{definition}
Let $c\in\mathbb{R}$. A sequence of random variables $W_0,\ldots,W_m$, with $m\ge1$, is a $c$-supermartingale with $b$-bounded increments when $W_0=0$ and for all $i\in[m]$: $\E[W_i|W_1,\ldots,W_{i-1}]\le W_{i-1}+c$, and $|W_i-W_{i-1}|\le b$. 
\end{definition}
The following general result may be regarded as a variation of Azuma's inequality.

\begin{lemma} \label{lemma:super-azuma}
If $W_0,\dots,W_m$ is a $c$-supermartingale with $b$-bounded increments then:
\[ \Pr\left[ W_m\geq mc + b\,\sqrt{2m \ln(1/\delta)} \right] \leq \delta ,\]
for all $0<\delta\le1$.
\end{lemma}
\begin{proof}
For any $\lambda > 0$,
\[\E[e^{\lambda W_m}]
= \E\left[ e^{\lambda W_{m-1}} \E \left[ e^{\lambda (W_m - W_{m-1})} \middle| W_1, \dots, W_{m-1} \right] \right]
\leq e^{\lambda c+\lambda^2 b^2 / 2} \E\left[ e^{\lambda W_{m-1}}\right], \]
by Hoeffding's Lemma. Hence, $\E[e^{\lambda W_m}]\le e^{\lambda m c+\lambda^2 m b^2 / 2}$ because $W_0=0$. In particular, for any $\lambda,t>0$, the Markov's inequality implies that
\[\Pr[ W_m \geq mc + t ]
=    \Pr\left[ e^{\lambda W_m} \geq e^{\lambda mc+\lambda t} \right]  \\
\leq e^{-\lambda m c-\lambda t} \E[e^{\lambda W_m}]\\
\leq e^{\lambda^2 m b^2 / 2 - \lambda t}.\]
For a given $t$, the right-hand side above is minimized for $\lambda=t/(mb^2)$, implying that
\[\Pr[ W_m \geq mc + t ]\le e^{\frac{-t^2}{2mb^2}}.\]
Letting $t=b\,\sqrt{2m\ln(1/\delta)}$ shows the lemma.
\end{proof}

Define $Z_{j,m}:=Z_m$, for $j\in[m]$, $X_m:=0$, $\hat S_m:=\sum_{j=1}^iZ_{j,m}$, and $E_m:=0$. Regarding the recursive construction, these random variables correspond to the base case with $i=m$. Next, for $i\in[m]$, define $(Z_{j,i-1}:j\in[i-1])$, $X_{i-1}$, $\hat S_{i-1}$, and $E_{i-1}$ recursively from $(Z_{j,i}:j\in[i])$, $X_i$, and $\hat S_i$ as follows (see Figure~\ref{fig:visualdef}). These constructions are such that for each $i=m,\ldots,0$:
\begin{itemize}
\item[(P1)] $\hat S_i=\sum\limits_{j=1}^iZ_{j,i}+\sum\limits_{j=i}^mX_j$. %
\item[(P2)] $(Z_{j,i}:j\in[i])$ has the same distribution as $(Z_j:j\in[i])$.
\end{itemize}

\begin{figure}
\centering
\begin{tabular}{|c|c|c|c|c|c|}
\hline
$i=4$ & $Z_{1,4}$ & $Z_{2,4}$ & $Z_{3,4}$ & $Z_{4,4}$ & $X_4$ \\
\hline
$i=3$ & $Z_{1,3}$ & $Z_{2,3}$ & $Z_{3,3}$ & $X_3$ & \textcolor{gray}{$X_4$} \\
\hline
$i=2$ & $Z_{1,2}$ & $Z_{2,2}$ & $X_2$ & \textcolor{gray}{$X_3$} & \textcolor{gray}{$X_4$} \\
\hline
$i=1$ & $Z_{1,1}$ & $X_1$ & \textcolor{gray}{$X_2$} & \textcolor{gray}{$X_3$} & \textcolor{gray}{$X_4$} \\
\hline
$i=0$ & $X_0$ & \textcolor{gray}{$X_1$} & \textcolor{gray}{$X_2$} & \textcolor{gray}{$X_3$} & \textcolor{gray}{$X_4$} \\
\hline
\end{tabular}
\caption{High-level visualization of our construction when $m=4$. The random variables in the array are constructed top-down, with dimmed ones constructed on earlier steps.}
\label{fig:visualdef}
\end{figure}
Since these properties hold for the base case with $i=m$, we may assume for the recursive construction that they also hold for $\hat S_i$ and $(Z_{j,i}:j\in[i])$. Let us now proceed with the construction.

Define $C_i:=B_i\cap[i-1]$ and $C_i:=\bar B_i\cap[i-1]$. Define
\[\bar\gamma_i:=(Z_{j,i}:j\in\bar C_i).\] 
Let $(Z_{j,i}':j\in C_i)$ be an independent draw from the conditional distribution of $(Z_{j,i}:j\in C_i)$ given $\bar\gamma_i$. In particular, we may rewrite property (P1) equivalently as follows:
\[\hat S_i=
\overbrace{\sum_{j\in C_i}Z_{j,i}'+\sum_{j\in\bar C_i}Z_{j,i}+\underbrace{Z_{i,i}+\sum_{j\in C_i}\big(Z_{j,i}-\E[Z_{j,i}'|\bar\gamma_i]\big)}_{\stackrel{\text{def}}{=}\,X_{i-1}}+\sum_{j=i}^m X_j}^{\stackrel{\text{def}}{=}\,\hat S_{i-1}}-\underbrace{\sum_{j\in C_i}\big(Z_{j,i}'-\E[Z_{j,i}'\mid\bar\gamma_i]\big)}_{\stackrel{\text{def}}{=}\,E_{i-1}},\]
where we have provided definitions for $\hat S_{i-1}$, $X_{i-1}$, and $E_{i-1}$. In particular, if we also define
\[Z_{j,i-1} := \begin{cases} Z_{j,i}', & j \in C_i\\ Z_{j,i}, & j \in \bar{C_i} \end{cases};\]
then property (P1) follows immediately for $\hat S_{i-1}$. On the other hand, from the definition of $(Z_{j,i}':j\in C_i)$, the joint distribution of $(Z_{j,i-1}:j\in C_i\cup\bar C_i)$ is the same as $(Z_{j,i}:j\in[i-1])$. In particular, by property (P2), $(Z_{j,i-1}:j\in[i-1])$ has the same distribution as $(Z_j:j\in[i-1])$, which shows (P2) for $\hat S_{i-1}$. 

We note that the definition of $(Z_{j,i}':j\in C_i)$ implies a third property:
\begin{itemize}
\item[(P3)] $(Z_{j,i-1} : j \in[i-1])$ and $(Z_{j,i'} : i' \ge i, j\in[i'])$ are conditionally independent given $\bar\gamma_i$.
\end{itemize}
We now use (P1)-(P3) to show some intermediate results that aid us in proving Theorem~\ref{thm:martingale}.

\begin{lemma} \label{lemma:cond-indep}
$X_i$ and $(X_0,\dots,X_{i-1})$ are conditionally independent given $\bar\gamma_{i+1}$. Also, $E_i$ and $(E_{i+1},\dots,E_m)$ are conditionally independent given $\bar\gamma_{i+1}$.
\end{lemma}

\begin{proof}
Since $X_m=0$ and $E_m=0$, we may assume without loss of generality that $1\le i<m$. 

Define $\textsc{High} = \big(Z_{j,i'} : i'\ge i+1, j\in[i']\big)$ and $\textsc{Low} = \big(Z_{j,i'} : i' \leq i, j\in[i']\big)$. From the recursive construction, $\textsc{Low}$ is a randomized function of $(Z_{j,i+1}':j\in C_{i+1})$ and $\bar\gamma_{i+1}$, which together form $(Z_{j,i}:j\in[i])$, which is conditionally independent of $\textsc{High}$ given $\bar\gamma_{i+1}$ by (P3). Consequently, $\textsc{Low}$ and $\textsc{High}$ are conditionally independent given $\bar\gamma_{i+1}$.

But note that $X_i$ and $(X_1,\dots,X_{i-1})$ are functions of $\textsc{High}$ and $\textsc{Low}$, respectively. In particular, they are conditionally independent given $\bar\gamma_{i+1}$.

On the other hand, $E_i$ is a function of $(Z_{j,i+1}' : j \in C_{i+1})$ and $\bar\gamma_{i+1}$, which together form the list $(Z_{j,i} : j \in [i])$, a sub-list of $\textsc{Low}$.
Similarly, $(E_{i+1},\dots,E_m)$ is a function of $\textsc{High}$, and the lemma follows arguing in the same way as we did in the previous paragraph.
\end{proof}

\begin{lemma} \label{lemma:X-props}
$X_i \in [1-b,b]$, $\E[X_i] = \E[Z_{i+1}]$, and $\big| \E[X_i \mid X_1,\dots,X_{i-1}] - \E[X_i] \big| \leq c$.
\end{lemma}

\begin{proof}
Since $X_m=0$, we may assume without loss of generality that $0\le i<m$. In particular:
\[X_i=Z_{i+1,i+1}+\sum_{j\in C_{i+1}}\big(Z_{j,i+1}-\E[Z_{j,i+1}'\mid\gamma_{i+1}]\big).\]
That is $X_i$ is the sum of at most $b$ variables in $[0,1]$, minus at most $b-1$ expectations in $[0,1]$; implying that $1-b\le X_i\le b$. 

Define $X:=(X_1,\ldots,X_{i-1})$. From Lemma~\ref{lemma:cond-indep}, $X_i$ and $X$ are conditionally independent given $\bar\gamma_{i+1}$. Hence, using a well-known property of conditional expectations, we find that
\begin{equation}\label{ide:useful}
\E[X_i \mid X_1,\dots,X_{i-1}]
=\E[X_i\mid X]
=\E[\E[X_i\mid \bar\gamma_{i+1},X]\mid X]
=\E[\E[X_i\mid \bar\gamma_{i+1}]\mid X].   
\end{equation}
But, from the definition of $X_i$:
\[\E[X_i\mid\bar\gamma_{i+1}]
=\E[Z_{i+1,i+1}\mid\bar\gamma_{i+1}]
+\sum_{j\in C_{i+1}}\big(\!\E[Z_{j,i+1}\mid\bar\gamma_{i+1}]-\E[Z_{j,i+1}'\mid \bar\gamma_{i+1}]\big).\]
From our recursive construction, however, $(Z_{j,i+1}':j\in C_{i+1})$ is a draw from the conditional distribution of  $(Z_{j,i+1}:j\in C_{i+1})$ given $\bar\gamma_{i+1}$. Hence, each term in the above summation above vanishes, implying that 
\begin{equation}\label{ide:useful2}
\E[X_i\mid\bar\gamma_{i+1}]=\E[Z_{i+1,i+1}\mid\bar\gamma_{i+1}].    
\end{equation}
In particular, $\E[X_i]=\E[Z_{i+1,i+1}]$, and (P2) implies that $\E[X_i]=\E[Z_{i+1}]$, as stated in the lemma. (P2) also implies that
\begin{align*}
\E[Z_{i+1,i+1} \mid \bar\gamma_{i+1}]
&= \E[Z_{i+1} \mid (Z_j : j \in \bar C_{i+1})]  \\
        &=    \E\left[ \E[Z_{i+1} \mid (Z_j : j \in\bar B_{i+1})] \mid (Z_j : j \in\bar C_{i+1}) \right]  \\
        &\leq    \E\left[ \E[Z_{i+1}] + c \mid (Z_j : j \in\bar C_{i+1}) \right]  \\
        &= \E[Z_{i+1}] + c,
\end{align*}
where for the above inequality we have use one of Theorem~\ref{thm:martingale}'s hypothesis. Finally, from (\ref{ide:useful})-(\ref{ide:useful2}), it follows that $\E[X_i\mid X_1,\dots,X_{i-1}]\leq\E[Z_{i+1}]+c$. Likewise, $\E[X_i\mid X_1,\dots,X_{i-1}]\geq\E[Z_{i+1}]-c$. Because we already showed $\E[Z_{i+1}] = \E[X_i]$, the lemma follows.
\end{proof}

This gives:
\begin{prop} \label{lemma:s-hat-expect}
$\Pr\left[ \left| \hat S_0 - \E[S] \right| \geq mc + b\sqrt{2m\ln(4/\delta)} \right] \leq \delta/2$.
\end{prop}

\begin{proof}
Define $W_i := \sum_{j=1}^i \big(X_i - \E[X_i]\big)$. From Lemma \ref{lemma:X-props}, $(W_i:i=0,\dots,m)$ is what we can call a $c$-supermartingale with $b$-bounded increments: $W_0 = 0$, $\E[W_i \mid W_0,\dots,W_{i-1}] \leq W_{i-1} + c$, and $|W_i - W_{i-1}| \leq b$.
Since $W_m = \hat S_0 - \E[S]$, Lemma \ref{lemma:super-azuma} implies that
\begin{equation}\label{ine:side1}
 \Pr\left[\hat S_0 - \E[S]\geq mc + b\sqrt{2m\ln(4/\delta)} \right] \leq \delta/4.   
\end{equation}
Likewise, since $(-W_i:i=0,\dots,m)$ is also a $c$-supermartingale with $b$-bounded increments, Lemma \ref{lemma:super-azuma} implies that
\begin{equation}\label{ine:side2}
\Pr\left[\E[S]-\hat S_0\geq mc + b\sqrt{2m\ln(4/\delta)} \right] \leq \delta/4.
\end{equation}
The proposition follows from equations (\ref{ine:side1})-(\ref{ine:side2}) by the sub-additive property of probabilities.
\end{proof}

Now we observe:
\begin{lemma} \label{lemma:E-props}
$E_i \in [-b,b]$ and $\E[E_i \mid E_{i+1},\dots,E_m] = 0$.
\end{lemma}

\begin{proof}
Since $E_m=0$, we may assume without loss of generality that $0\le i<m$. In this case:
\[E_i=\sum_{j\in C_{i+1}}\big(Z_{j,i+1}'-\E[Z_{j,i+1}'\mid\bar\gamma_{i+1}]\big).\]
Hence, $|E_i|\le b$, by similar arguments as we have given before. 

Define $E:=(E_{i+1},\ldots,E_m)$. Observe that
\[\E[E_i \mid E_{i+1},\dots,E_m]
= \E[E_i|E]
= \E \big[ \E[E_i | E, \bar\gamma_{i+1}] \mid E\big]
=\E\big[\E[E_i|\bar\gamma_{i+1}] \mid E\big], \]
where for the last identity we have used Lemma \ref{lemma:cond-indep}. But: 
\[\E[E_i \mid \bar\gamma_{i+1}]=\sum_{j \in C_i} \left( \E[Z'_{j,i+1} \mid \bar\gamma_{i+1}] - \E[Z'_{j,i+1} ] \mid \bar\gamma_{i+1} \right)=0,\]
hence $\E[E_i \mid E_{i+1},\dots,E_m]=0$, which shows the lemma.
\end{proof}

This gives:
\begin{prop} \label{lemma:s-hat-s}
    $\Pr\left[ \big| \hat S_0 - S \big| \geq b\sqrt{2m \ln(4/\delta)} \right] \leq \delta/2$.
\end{prop}

\begin{proof}
Define $W_i := \sum_{j=i}^m E_i$. From Lemma~\ref{lemma:E-props}, $(W_i:i=m,\dots,0)$ is a $0$-supermartingale (i.e. a martingale) with $b$-bounded increments. Since $W_1 = \hat S_0 - S$, Lemma \ref{lemma:super-azuma} implies that $\Pr(\hat S_0-S\ge b\sqrt{2m\ln(4/\delta)})\le\delta/4$. Likewise, because $(-W_i:i=m,\dots,0)$ is a martingale with $b$-bounded increments, $\Pr(S-\hat S_0\ge b\sqrt{2m\ln(4/\delta)})\le\delta/4$, from which the proposition follows.
\end{proof}

Theorem~\ref{thm:martingale} is now a direct consequence of Propositions \ref{lemma:s-hat-expect} and \ref{lemma:s-hat-s}.
\end{proof}

\end{document}